\documentclass[journal]{IEEEtran}
\usepackage{balance}

\usepackage{graphicx}

\usepackage{epstopdf} 

\usepackage{epsfig}
\graphicspath{ {fig/} }
\ifCLASSINFOpdf
\else
\fi
\usepackage{amsmath}
\usepackage{amssymb}
\usepackage{amsthm}

\theoremstyle{definition}
\newtheorem{definition}{Definition}
\newtheorem{theorem}{Theorem}

\usepackage[inline]{enumitem}
\newlist{inlinelist}{enumerate*}{1}
\setlist*[inlinelist,1]{%
  label=(\roman*),
}

\usepackage{booktabs}
\usepackage{array}
\newcolumntype{L}[1]{>{\raggedright\let\newline\\\arraybackslash\hspace{0pt}}m{#1}}
\newcolumntype{C}[1]{>{\centering\let\newline\\\arraybackslash\hspace{0pt}}m{#1}}
\newcolumntype{R}[1]{>{\raggedleft\let\newline\\\arraybackslash\hspace{0pt}}m{#1}}

\begin{document}
\title{Less Is More: A Comprehensive Framework for \\the Number of Components of Ensemble Classifiers}
%
%
%

\author{Hamed~Bonab,
        and Fazli~Can
\thanks{This is an extended version of the work presented as a short paper at the Conference on Information and Knowledge Management (CIKM), 2016 \cite{bib:bonabcikm}.}
\thanks{Manuscript submitted July 2018.}}
%
%

\markboth{IEEE TRANSACTIONS ON NEURAL NETWORKS AND LEARNING SYSTEMS,~VOL.~14, No.~8, July~2018}%
{Shell \MakeLowercase{\textit{et al.}}: Bare Demo of IEEEtran.cls for IEEE Journals}

\maketitle

\begin{abstract}
The number of component classifiers chosen for an ensemble greatly impacts the prediction ability. In this paper, we use a geometric framework for \textit{a priori} determining the ensemble size, which is applicable to most of existing batch and online ensemble classifiers. There are only a limited number of studies on the ensemble size examining Majority Voting (MV) and Weighted Majority Voting (WMV). Almost all of them are designed for batch-mode, hardly addressing online environments. Big data dimensions and resource limitations, in terms of time and memory, make determination of ensemble size crucial, especially for online environments. For the MV aggregation rule, our framework proves that the more strong components we add to the ensemble, the more accurate predictions we can achieve. For the WMV aggregation rule, our framework proves the existence of an ideal number of components, which is equal to the number of class labels, with the premise that components are completely independent of each other and strong enough. While giving the exact definition for a strong and independent classifier in the context of an ensemble is a challenging task, our proposed geometric framework provides a theoretical explanation of diversity and its impact on the accuracy of predictions. We conduct a series of experimental evaluations to show the practical value of our theorems and existing challenges. 
\end{abstract}

\begin{IEEEkeywords}
Supervised learning, ensemble size, ensemble cardinality, voting framework, big data, data stream, law of diminishing returns
\end{IEEEkeywords}

\IEEEpeerreviewmaketitle

\section{Introduction}
\IEEEPARstart{O}{ver} the last few years, the design of learning systems for mining the data generated from real-world problems has encountered new challenges such as the high dimensionality of big data, as well as growth in volume, variety, velocity, and veracity---the four V's of big data\footnote{http://www.ibmbigdatahub.com/infographic/four-vs-big-data}. In the context of data dimensions, volume is the amount of data, variety is the number of types of data, velocity is the speed of data, and veracity is the uncertainty of data; generated in real-world applications and processed by the learning algorithm. The dynamic information processing and incremental adaptation of learning systems to the temporal changes are among the most demanding tasks in the literature for a long time \cite{can1993incremental}.  



Ensemble classifiers are among the most successful and well-known solutions to supervised learning problems, particularly for online environments \cite{dietterich2000ensemble,bib:gama2010book,bib:wang2015}. The main idea is to construct a collection of individual classifiers, even with weak learners, and combine their votes. The aim is to build a stronger classifier, compared to each individual component classifier \cite{bib:surveygama2017}. The training mechanism of components and the vote aggregation method mostly characterize an ensemble classifier \cite{bib:rokach2010ensemble}.  

There are two main categories of vote aggregation methods for combining the votes of component classifiers: weighting methods and meta-learning methods \cite{bib:rokach2010ensemble,bib:chan1999weighted}. Weighting methods assign a combining weight to each component and aggregate their votes based on these weights (e.g. Majority Voting, Performance Weighting, Bayesian Combination). They are useful when the individual classifiers perform the same task and have comparable success. Meta-learning methods refer to learning from the classifiers and from the classifications of these classifiers on training data (e.g. Stacking, Arbiter Trees, Grading). They are best suited for situations where certain classifiers consistently mis-classify or correctly classify certain instances \cite{bib:rokach2010ensemble}. In this paper we study the ensembles with the weighting combination rule. Meta-learning methods are out of the scope of this paper. 

An important aspect of ensemble methods is to determine how many component classifiers should be included in the final ensemble, known as the ensemble size or ensemble cardinality \cite{bib:bonabcikm,bib:rokach2010ensemble,bib:latinne2001limiting,hu2001using,bib:oshiro2012many,bib:hernandez2013large,bib:surveybifet2017}. The impact of ensemble size on \textit{efficiency} in terms of time and memory and \textit{predictive performance} make its determination an important problem \cite{tsoumakas2008taxonomy,gomes2014mining}. Efficiency is especially important for online environments. In this paper, we extend our geometric framework \cite{bib:bonabcikm} for pre-determining the ensemble size, applicable to both batch and online ensembles. 

Furthermore, diversity among component classifiers is an influential factor for having an accurate ensemble \cite{bib:rokach2010ensemble,bib:jackowski2018new,kuncheva2004combining,kuncheva2003measures}. Liu et al. \cite{liu2004empirical} empirically studied ensemble size on diversity. Hu \cite{liu2004empirical} explained that component diversity leads to uncorrelated votes, which in turn improves predictive performance. However, to the best of our knowledge, there is no explanatory theory revealing how and why diversity among components contributes toward overall ensemble accuracy \cite{brown2005diversity}. Our proposed geometric framework introduces a theoretical explanation for the understanding of diversity in the context of ensemble classifiers.

\noindent The main contributions of this study are the following. We
\begin{itemize}
    \item Present a brief comprehensive review of existing approaches for determining the number of component classifiers of ensembles,   
    \item Provide a spatial modeling for ensembles and use the linear least squares (LSQ) solution \cite{bib:lsqbook} for optimizing the weights of components of an ensemble classifier, applicable to both online and batch ensembles, 
    \item Exploit the geometric framework for the first time in the literature, for \textit{a priori} determining the number of component classifiers of an ensemble,
    \item Explain the impact of diversity among component classifiers of an ensemble on the predictive performance, from a theoretical perspective and for the first time in the literature,
    \item Conduct a series of experimental evaluations on more than 16 different real-world and synthetic data streams and show the practical value of our theorems and existing challenges.  
\end{itemize}

\section{Related Works}
The importance of ensemble size is discussed in several studies. There are two categories of approaches in the literature for determining ensemble size. Several ensembles \textit{a priori} determine the ensemble size with a fixed value (like bagging), while others try to determine the best ensemble size dynamically by checking the impact of adding new components to the ensemble \cite{bib:rokach2010ensemble}. Zhou et al. \cite{zhou2002ensembling} analyzed the relationship between an ensemble and its components, and concluded that aggregating \textit{many} of the components may be the better approach. Through an empirical study, Liu et al. \cite{liu2004empirical} showed that a subset of the components of a larger ensemble can perform comparably to the full ensemble, in terms of accuracy and diversity. Ula{\c{s}} et al. \cite{ulacs2009incremental} discussed approaches for incrementally constructing a batch-mode ensemble using different criteria including accuracy, significant improvement, diversity, correlation, and the role of search direction. 

This led to the idea in ensemble construction, that it is sometimes useful, to let the ensemble extend unlimitedly and then prune the ensemble in order to get a more effective ensemble \cite{bib:rokach2010ensemble,rokach2009collective,margineantu1997pruning,toraman2012,bib:nse,windeatt2013ensemble}. Ensemble selection methods developed as pruning strategies for ensembles. However, with today's data dimensions and resource constraints, the idea seems impractical. Since the number of data instances grow exponentially, especially in online environments, there is a potential problem of approaching an infinite number of components for an ensemble. As a result, determining an upper bound for the number of components with a reasonable resource consumption is essential. As mentioned in \cite{bib:decbound}, the errors cannot be arbitrarily reduced by increasing the ensemble size indefinitely.

There are a limited number of studies for batch-mode ensembles. Latinne et al. \cite{bib:latinne2001limiting} proposed a simple empirical procedure for limiting the number of classifiers based on the McNemar non-parametric test of significance. Similar approaches \cite{bib:fumera-ref2,bib:fumera}, suggested a range of 10 to 20 base classifiers for bagging, depending on the base classifier and dataset. 

Oshiro et al. \cite{bib:oshiro2012many} cast the idea that there is an ideal number of component classifiers within an ensemble. They defined the ideal number as the ensemble size where exploiting more base classifiers brings no significant performance gain, and only increases computational costs. They showed this by using the weighted average area under the ROC curve (AUC), and some dataset density metrics. Fumera et al. \cite{bib:fumera,bib:fumera-ref8} applied an existing analytical framework for the analysis of linearly combined classifiers of bagging, using misclassification probability. Hern\'{a}ndez-Lobato et al. \cite{bib:hernandez2013large} suggested a statistical algorithm for determining the size of an ensemble, by estimating the required number of classifiers for obtaining stable aggregated predictions, using majority voting.

Pietruczuk et al. \cite{pietruczuk2016method,pietruczuk2017adjust} recently studied the automatic adjustment of ensemble size for online environments. Their approach determines whether a new component should be added to the ensemble by using a probability framework and defining a confidence level. However, the diversity impact of component classifiers is not taken into account, and there is a possibility of approaching to the infinite number of components without reaching the confidence level. The assumption that the error distribution is i.i.d can not be guaranteed, especially with a higher ensemble size; this reduces the improvements due to each extra classifier \cite{bib:decbound}.

\begin{table}[t]
    \centering
	\caption{Symbol Notations of the Geometric Framework }
		\begin{tabular}{L{3.41cm}@{\hspace{0.5\tabcolsep}}L{5cm}} 
			\toprule 
			\textbf{Notation} & \textbf{Definition}  \\ 
			\toprule 
			
			$ I=\{I_1, I_2, ..., I_n\} $ 	& Instance window, \(I_i; (1 \leq i \leq n)\) \\
			\midrule
			
			$\xi=\{CS_1, CS_2, ..., CS_m\}$ & Ensemble of $m$ component classifiers, \(CS_j; (1 \leq j \leq m\) and $2 \leq m)$\\
			\midrule
			
			$ C=\{C_1, C_2, ..., C_p\} $    & Classification problem with $p$ class labels, \(C_k; (1 \leq k \leq p\) and $2 \leq p)$\\
			\midrule 
			
			\( s_{ij}=<S_{ij}^1, S_{ij}^2, ..., S_{ij}^p> \) & Score-vector for \( I_i \) and \( CS_j \) \\
			\midrule

			\( o_i = <O^1_i, O^2_i, \cdots, O^p_i> \) & Ideal-point for \( I_i \), \(O^k_i; (1 \leq k \leq p)\) \\	
            \midrule
            
			\( a_i = <A^1_i, A^2_i, \cdots, A^p_i> \) & Centroid-point for \( I_i \), \(A^k_i; (1 \leq k \leq p)\) \\	
            \midrule
            
			\(w = <W_1, W_2, \cdots, W_m> \) & Weight vector for \( \xi \), \(W_j; (1 \leq j \leq m)\) \\	  
			\midrule
            
			\( b_i = <B^1_i, B^2_i, \cdots, B^p_i> \) & Weighted-centroid-point for \( I_i \), \(B^k_i; (1 \leq k \leq p)\) \\
            
			\bottomrule 
		\end{tabular}
	\label{tab:notation}
\end{table}

\section{A Geometric Framework}
In this section we propose a geometric framework for studying the theoretical side of ensemble classifiers based on \cite{bib:bonabcikm}. We mainly focus on online ensembles since they are more specific models compared to batch-mode ensembles. The main difference is that online ensembles are trained and tested over the course of incoming data while batch-mode ensembles are trained and tested once. As a result, batch-mode ensembles are also applicable to our framework, with a simpler declaration. We use this geometric framework in \cite{bib:bonabtkdd} for aggregating votes and proposing a novel online ensemble for evolving data stream, called GOOWE.

Suppose we have an ensemble of $m$ component classifiers, \(\xi = \{CS_1, CS_2, \cdots , CS_m\}\). Due to resource limitations, we are only able to keep the $n$ latest instances of an incoming data stream as an \emph{instance window}, \(I = \{I_1, I_2, \cdots, I_n\}\), where \(I_n\) is the latest instance and all the true-class labels are available. We assume that our supervised learning problem has $p$ class labels, \(C = \{C_1, C_2, \cdots, C_p\}\). For batch-mode ensembles, $I$ can be considered as the whole set of training data. Table \ref{tab:notation} presents the notation of symbols for our geometric framework.

Our framework uses a $p$-dimensional Euclidean space for modeling the components' votes and true class labels. For a given instance \(I_i (1 \leq i \leq n)\), each component classifier \( CS_j (1 \leq j \leq m)\) returns a \textit{score-vector} as \( s_{ij} = < S^1_{ij}, S^2_{ij}, \cdots, S^p_{ij} >\), where \( \sum_{k=1}^{p} S^k_{ij} = 1 \). Considering all the score-vectors in our $p$-dimensional space, the framework builds a polytope of votes which we call the \textit{score-polytope} of \(I_i\). For the true-class label of \(I_i\) we have \( o_i = <O^1_i, O^2_i, \cdots, O^p_i> \) as the \textit{ideal-point}. Note that $o_i$ is a \emph{one-hot} vector in this study. However, there are other supervised problems this assumption is not true---e.g. multi-label classification. Studying other variations of the supervised problem is out of the scope of this study. A general schema of our geometric framework is presented in Fig. \ref{fig:schema}.

\textit{Example.} Assume we have a supervised problem with 3 class labels, \(C = \{C_1, C_2, C_3\}\). For a given instance \(I_t\), the true-class label is \(C_2\). The ideal-point would be \(o_t = <0, 1, 0>\).

\begin{figure}
	\includegraphics[width=\linewidth]{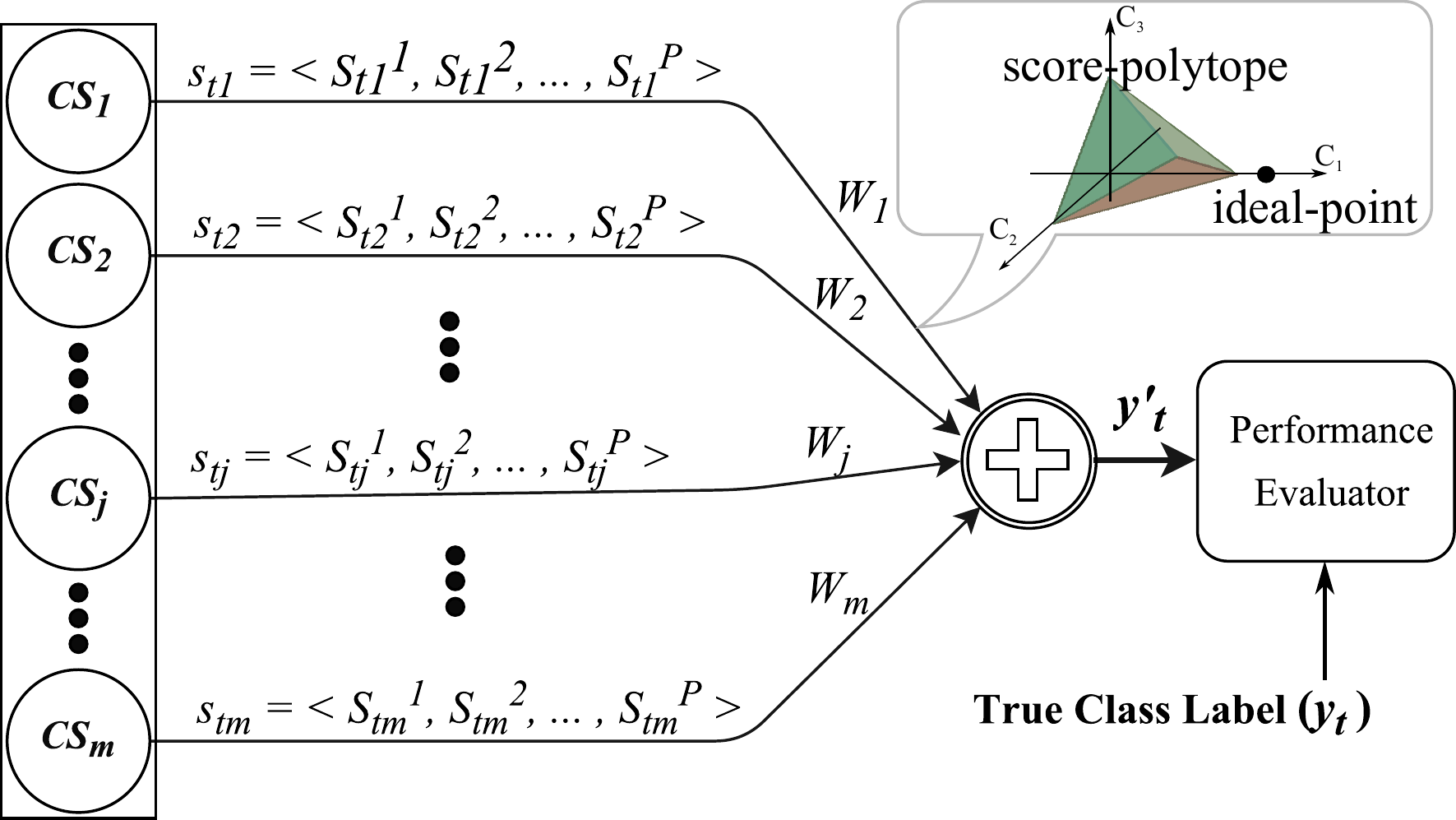}
	\caption{Schema of the geometric framework (obtained from \cite{bib:bonabcikm}). $I_t$ with class label $y_t = C_1$ is fed to the ensemble. Each component classifier, $CS_j$, generates a score-vector, $S_{tj}$. These score-vectors construct a surface in the Euclidean space, called \textit{score-polytope}.} 
	\label{fig:schema}
\end{figure}

One could presumably define many different algebraic rules for vote aggregation \cite{bib:dwm,bib:dwm2}---minimum, maximum, sum, mean, product, median, etc. While these vote aggregation rules can be expressed using our geometric framework, we study the Majority Voting (MV) and Weighted Majority Voting (WMV) aggregation rules in this paper. In addition, individual vote scores can be aggregated based on two different voting schemes \cite{miller1999critic}; \textit{(a) Hard Voting}, the score-vector of a component classifier is first transformed into a one-hot vector, possibly using a hard-max function, and then combined, \textit{(b) Soft Voting,} the score-vector is used for vote aggregation. We use soft voting for our framework.    

The Euclidean norm is used as the loss function, $loss(\cdotp,\cdotp)$, for optimization purposes \cite{bib:lsqbook}. The Euclidean distance of any score-vector and ideal-point expresses the effectiveness of that component for the given instance. Using aggregation rules, we aim to define a mapping function from a score-polytope into a single vector, and measure the effectiveness of our ensemble. Wu and Crestani \cite{wu2015geometric} applied a similar geometric framework for data fusion of information retrieval systems. In this study, some of our theorems are obtained and adapted to ensemble learning from their framework.

\subsection{Majority Voting (MV)}
The mapping of a given score-polytope into its centroid can be expressed as the MV aggregation---plurality voting or averaging. For a given instance, \(I_t\),  we have the following mapping to the \textit{centroid-point}, \( a_t = <A^1_t, A^2_t, \cdots, A^p_t> \).
\begin{equation}
A^k_t = \frac{1}{m} \sum_{j=1}^{m} S^k_{tj} ~~ (1 \leq k \leq p)
\label{eq:centroid}
\end{equation}

\begin{theorem}
	For \(I_t\) the loss between the centroid-point $a_t$ and ideal-point $o_t$ is not greater than the average loss between $m$ score-vectors and $o_t$, that is to say,
	\begin{equation}
    loss(a_t, o_t) \leq \frac{1}{m} \sum_{j=1}^{m} loss(s_{tj}, o_t)
	\label{eq:theorem1}
	\end{equation} 
	\label{th:theorem1}
\end{theorem}
\begin{proof}
	Based on Minkowski's inequality for sums \cite{abramowitz1964handbook},  
	$$ \sqrt{\sum_{k=1}^{p} (\sum_{j=1}^{m} \theta_j^k )^2 }  \leq  \sum_{j=1}^{m} \sqrt{\sum_{k=1}^{p} (\theta_j^k)^2 } $$
	Letting $ \theta_j^k = S^k_{tj} - O^k_t $ and substituting results   
	$$ \sqrt{\sum_{k=1}^{p} ( m ( \dfrac{1}{m} \sum_{j=1}^{m} (S^k_{tj} - O^k_t ) ) )^2 }  \leq  \sum_{j=1}^{m} \sqrt{\sum_{k=1}^{p} (S^k_{tj} - O^k_t )^2 } 	$$	
	Since $ m > 0 $, we have the following, 
	$$ \sqrt{\sum_{k=1}^{p} ( \dfrac{1}{m} \sum_{j=1}^{m} S^k_{tj} - O^k_t )^2 }  \leq \dfrac{1}{m}  \sum_{j=1}^{m} \sqrt{\sum_{k=1}^{p} (S^k_{tj} - O^k_t)^2 } 	$$
    Using Eq. \ref{eq:centroid} and $loss$ definition, Eq. \ref{eq:theorem1} can be achieved.
\end{proof}
\textit{\textbf{Discussion.}} Theorem 1 shows that the performance of an ensemble with the MV aggregation rule is at least equal to the average performance of all individual components.  

\begin{theorem}
	For \(I_t\), let \(\xi_l =  \xi - \{CS_l\} ~(1 \leq l \leq m)\) be a subset of ensemble \(\xi\) without $CS_l$. Each \(\xi_l\) has $a_{tl}$ as its centroid-point. We have      
	\begin{equation}
	loss(a_t, o_t) \leq \frac{1}{m} \sum_{l=1}^{m} loss(a_{tl}, o_t)
	\label{eq:theorem2}
	\end{equation}
	\label{th:theorem2}
\end{theorem}
\begin{proof}
$a_{t}$ is the centroid-point of all $a_{tl}$ points according to the definition. Assume $\xi_l$ as an individual classifier with score-vector of $a_{tl}$. Theorem \ref{th:theorem1} for every $\xi_l$ results Eq. \ref{eq:theorem2} directly.  
\end{proof}

\textit{\textbf{Discussion.}} Theorem 2 can be generalized for any subset definition with $(m-f)$ components, $1 \leq f \leq (m-2)$. This shows that an ensemble with $m$ components performs better (or at least equal) compared to the average performance of ensembles with $m-f$ components. It can be concluded that better performance can be achieved if we aggregate more component classifiers. However, if we keep adding poor components to the ensemble, it can diminish overall prediction accuracy by increasing the upper bound in Eq. \ref{eq:theorem1}. This is in agreement with the result of the Bayes error reduction analysis \cite{tumer1996error,bib:awe}. Setting a threshold, as expressed in \cite{bib:decbound,pietruczuk2016method,pietruczuk2017adjust,bib:awe}, can give us the ideal number of components for a specific problem.

\subsection{Weighted Majority Voting (WMV)}
For this aggregation rule, a weight vector \(w = <W_1, W_2, \cdots, W_m> \) for components of ensemble is defined, $W_j \ge 0$ and $\sum W_j = 1$ for $1 \leq j \leq m$. For a given instance, \(I_t\), we have the following mapping to the \textit{weighted-centroid-point}, \( b_t = <B^1_t, B^2_t, \cdots, B^p_t> \).   
\begin{equation}
B^k_t = \sum_{j=1}^{m} W_j S^k_{tj} ~~ (1 \leq k \leq p)
\end{equation}

Note that giving equal weights to all the components will result in the MV aggregation rule. WMV presents a flexible aggregation rule. No matter how poor a component classifier is, with a proper weight vector we can cancel its effect on the aggregated results. However, as discussed earlier, this is not true for the MV rule. In the following, we give the formal definition of the optimum weight vector, which we aim to find.

\theoremstyle{definition}
\begin{definition}{\textit{Optimum Weight Vector.}}
	For an ensemble, \(\xi\), and a given instance, \(I_t\), weight vector \(w_o\) with the weighted-centroid-point \(b_o\) is the optimum weight vector where for any \(w_x\) with weighted-centroid-point \(b_x\) the following is true; \( loss(b_o, o_t) \le loss(b_x, o_t)  \).
\end{definition}

\begin{theorem}
	For a given instance, \(I_t\), let the optimum weight vector, $w_o$, and the weighted-centroid-point \(b_t\). The following must hold;
	\begin{equation}
	loss(b_t, o_t) \leq \min \{ loss(s_{t1}, o_t), \ldots, loss(s_{tm}, o_t) \}
	\label{eq:theorem3}
	\end{equation} 	
\end{theorem}
\begin{proof}
  Assume that the least \textit{loss} belongs to component $j$, among $m$ score-vectors. We have the following two cases. 
  
  \noindent \textit{(a) Maintaining the performance.} Simply giving a weight of $1$ to $j$'s component and $0$ for the remaining components result in the equality case; $loss(b_t, o_t) = loss(s_{tj}, o_t)$.
  
  \noindent \textit{(b) Improving the performance.} Using a linear combination of $j$ and other components with proper weights result in a weighted-centroid-point closer to the ideal-point. We can always find such a weight vector in the Euclidean space if other components are not the exact same as $j$. 
\end{proof}

Using the squared Euclidean norm as the measure of closeness for the linear least squares problem (LSQ) \cite{bib:lsqbook} results 
\begin{equation}
\min_w || o -  wS ||^2_2
\end{equation} 
Where for each instance \( I_i\) in the instance window, \( S \in \mathbb{R}^{m \times p} \) is the matrix with score-vectors \(s_{ij}\) in each row corresponding to the component classifier $j$, \(w \in \mathbb{R}^{m}\) is the vector of weights to be determined, and \(o \in \mathbb{R}^p\) is the vector of the ideal-point. We use the following function for our optimization solution.  
\begin{equation}
f(W_1, W_2, \cdots, W_m) =  \sum_{k=1}^{p} (\sum_{j=1}^{m} (W_j S^k_{ij}) - O^k_i)^2
\end{equation}
Taking a partial derivation over \(W_q (1 \leq q \leq m)\), setting the gradient to zero, \( \nabla f = 0 \), and finding optimum points give us the optimum weight vector. Letting the following summations as \(\lambda_{qj}\) and \(\gamma_q\) 

\begin{equation} \label{eq:aqj1}
\lambda_{qj} = \sum_{k=1}^{p} S^k_{iq} S^k_{ij} ,~~(1 \leq q, j \leq m)
\end{equation}
\begin{equation}\label{eq:dq1}
\gamma_q =  \sum_{k=1}^{p} O^k_i S^k_{iq} ,~~(1 \leq q \leq m)
\end{equation}
lead to \(m\) linear equations with \(m\) variables (weights). Briefly, $w\Lambda=\gamma$, where \( \Lambda \in \mathbb{R}^{m \times m} \) is the coefficients matrix and \(\gamma \in \mathbb{R}^{m}\) is the remainders vector---using Eq. \ref{eq:aqj1} and Eq. \ref{eq:dq1}, respectively. The proper weights in the matrix equation are our intended optimum weight vector. For a more detailed explanation see \cite{bib:bonabtkdd}.  

\textit{\textbf{Discussion.}} According to Eq. \ref{eq:aqj1}, \( \Lambda \) is a symmetric square matrix. If \( \Lambda \) has full rank, our problem has a unique solution. On the other hand, in the sense of a least squares solution \cite{bib:lsqbook}, it is probable that \( \Lambda \) is rank-deficient, and we may not have a unique solution. Studying the properties of this matrix lead us to the following theorem.  

\begin{theorem}
	If the number of component classifiers is not equal to the number of class labels, $m \neq p$, then the coefficient matrix would be rank-deficient, $ \det{\Lambda} = 0$.  
	\label{th:theorem4}
\end{theorem}
\begin{proof}
	Since we have $p$ dimensions in our Euclidean space, $p$ independent score-vectors would be needed for the basis spanning set. Any number of vectors, $m$, more than $p$ is dependent on the basis spanning set, and any number of vectors, $m$, less than $p$ is insufficient for constructing the basis spanning set.  
\end{proof}

\textit{\textbf{Discussion.}} The above theorem excludes some cases in which we cannot find optimum weights for aggregating votes. There are several numerical solutions for solving rank-deficient least-squares problems (e.g. QR factorization, Singular Value Decomposition (SVD)), however the resulting solution is relatively expensive, may not be unique, and optimality is not guaranteed. Theorem 4's outcome is that the number of independent components for an ensemble is crucial for providing full-rank coefficient matrix, in the aim of an optimal weight vector solution.

\begin{figure*}
	\includegraphics[width=\textwidth]{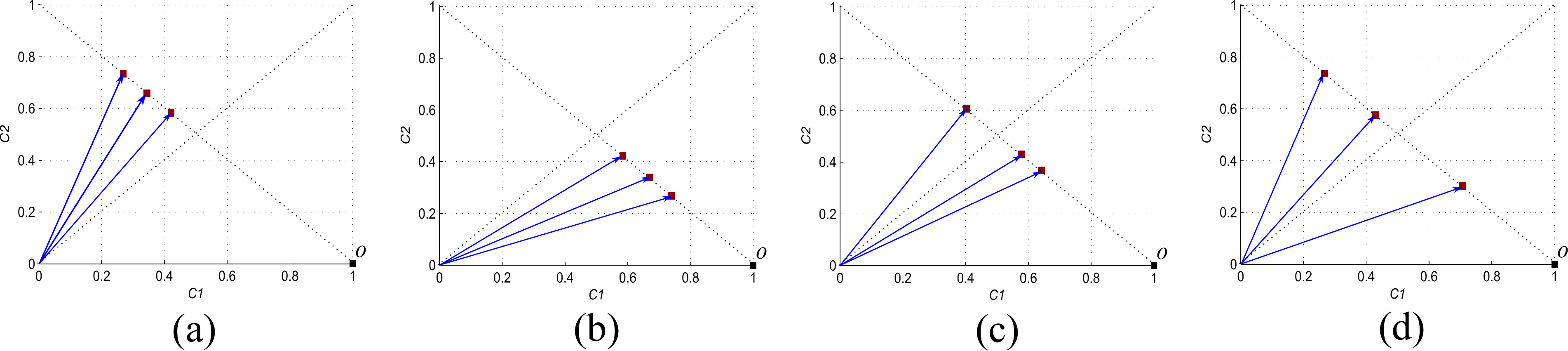}
	\caption{Four score-vector possibilities of an ensemble with size three. The true class label of the instance is $C1$ for a binary classification problem. If two of these score-vectors exactly match each other for several data instances, we cannot consider them to be independent and diverse enough components.  } 
	\label{fig:div}
\end{figure*}

\subsection{Diversity among components}
Theorem \ref{th:theorem4} shows that for weight vector optimality, $m = p$ should be true. However, the reverse cannot be guaranteed in general. Assuming $m = p$ and letting $\det \Lambda = 0$ for the parametric coefficient matrix results in some conditions where we have vote agreement, and cannot find a unique optimum weight vector. As an example, suppose we have two component classifiers for a binary classification task, $m = p = 2$. Letting $\det \Lambda = 0$, results the following equations;
$
S^1_{11} + S^2_{12} = 1$ or $
S^2_{11} + S^1_{12} = 1
$, meaning the agreement of component classifiers---i.e. the exact same vote vectors. More specifically, this suggests another condition for weight vector optimality: the importance of diversity among component classifiers.

Fig. \ref{fig:div} presents four mainly different score-vector possibilities for an ensemble with size three. The true class label of the examined instance is $C1$ for a binary classification problem. All score-vectors are normalized and placed on the main diagonal of the spatial environment. The counter-diagonal line divides the decision boundary for the class label determination based on the probability values. If the component's score-vector is in the lower triangular, it is classified $C1$ and similarly, if it is in the left triangular part it is classified $C2$. Fig. \ref{fig:div} (a) and (b) show the miss-classification and true classification situations, respectively. Fig. \ref{fig:div} (c) and (d) show disagreement among components of the ensemble. 

If for several instances, in a sequence of data, the score-vectors of two components are equal (or act predictably similar), they are considered dependent components. There are several measurements for quantifying this dependency for ensemble classifiers (e.g. Q-statistic) \cite{kuncheva2003measures}. However, most of the measurements in practice use the oracle output of components (i.e. only predicted class labels) \cite{kuncheva2003measures}. Our geometric framework shows a potential importance of using score-vectors for diversity measurements. It is out of the scope of this study to propose a diversity measurement using score-vectors and we leave it as a future work.

To the best of our knowledge, there is no explanatory theory in the literature revealing why and how diversity among components contribute toward overall ensemble accuracy \cite{brown2005diversity,bib:jackowski2018new}. Our geometric modeling of ensemble's score-vectors and the optimum weight vector solution provide a theoretical insight for the commonly agreed upon idea that ``the classifiers should be different from each other, otherwise the overall decision will not be better than the individual decisions'' \cite{kuncheva2003measures}. Optimum weights can be reached when we have the same number of independent and diverse component classifiers as class labels. Diversity has a great impact on the coefficient matrix that consequently impacts the accurate predictions of an ensemble. For the case of majority voting, adding more dependent classifiers will dominate the decision of other components.

\textit{\textbf{Discussion.}} Our geometric framework supports the idea that there is an ideal number of component classifiers for an ensemble, with which we can reach the most accurate results. Increasing or decreasing the number of classifiers from this ideal point may deteriorate predictions, or bring no gain to the overall performance of the ensemble. Having more components than the ideal number of classifiers can mislead the aggregation rule, especially for majority voting. On the other hand, having fewer is insufficient for constructing an ensemble which is stronger than the single classifier. We refer to this situation as ``the law of diminishing returns in ensemble construction.'' 

Our framework suggests the number of class labels of a dataset as the ideal number of component classifiers, with the premise that they generate independent scores and aggregated with optimum weights. However, real-world datasets and existing ensemble classifiers do not guarantee this premise most of the time. Determining the exact value of this ideal point for a given ensemble classifier, over real-world data, is still a challenging problem due to the different complexities of datasets.

\section{Experimental Evaluation}
 The experiments conducted in \cite{bib:bonabcikm} showed that for ensembles trained with a specific dataset, we have an ideal number of components in which having more will deteriorate or at least provide no benefit to our prediction ability. Our extensive experiments in \cite{bib:bonabtkdd} show the practical value of this geometric framework for aggregating votes. 
 
 Here, through a series of experiments, we first investigate the impact of the number of class labels and the number of component classifiers for MV and WMV using a synthetic dataset generator; Then, we study impact of miscellaneous data streams using several real-world and synthetic datasets; Finally, we explore the outcome of our theorems on the diversity of component classifiers and the practical value of our study. All the experiments are implemented using the MOA framework \cite{bib:moa}, and Interleaved-Test-Then-Train is used for accuracy measurements. An instance window, with a length of $500$ instances, is used for keeping the latest instances.

 \begin{figure*}
	\includegraphics[width=\textwidth]{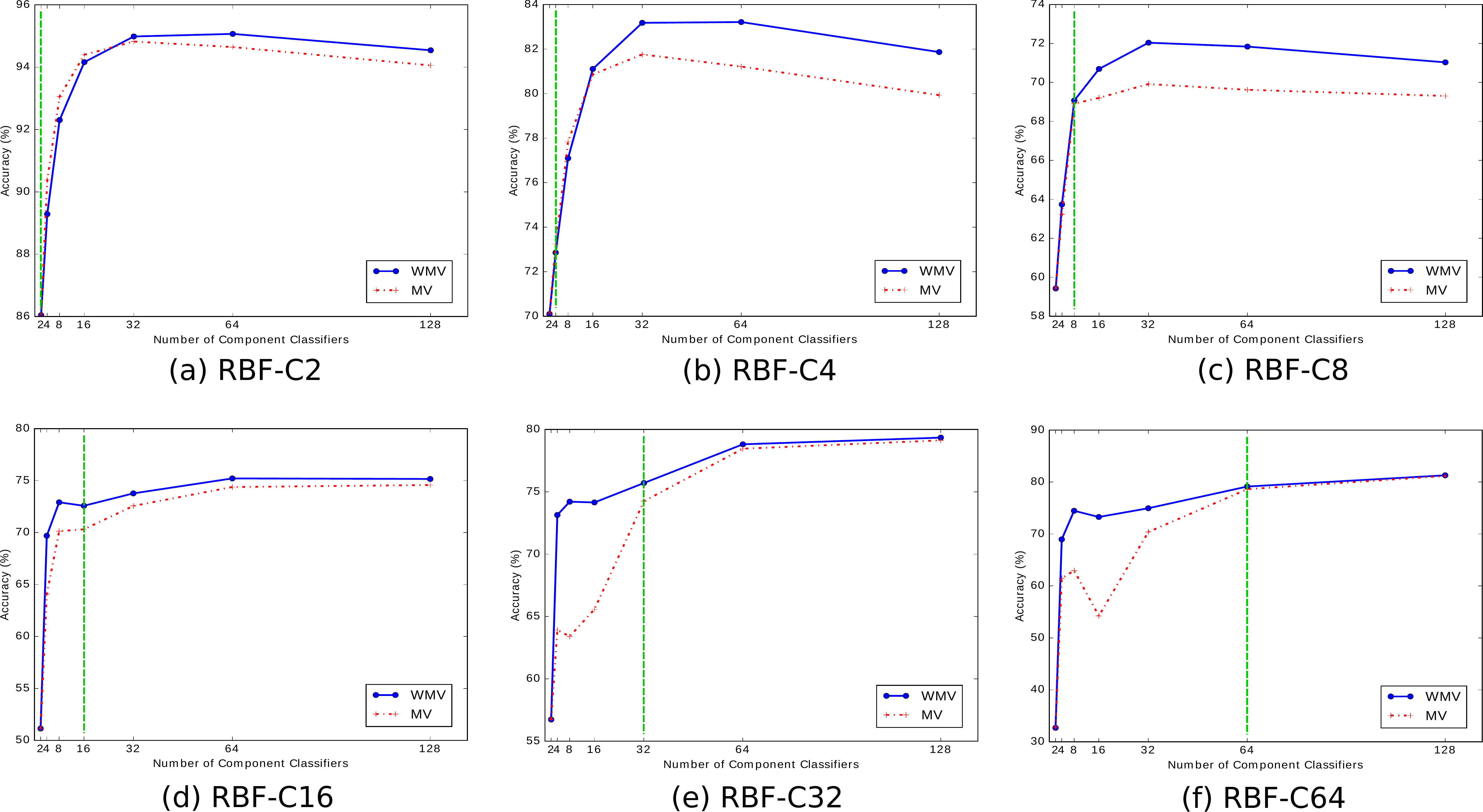}
	\caption{ Prediction behavior of WMV and MV aggregation rules, in terms of accuracy, for RBF-C datasets with increasing both the number of component classifiers, $m$, and the number of class labels, $p$. The equality case, $m=p$, is shown on each plot using a vertical dashed green line.  } 
	\label{fig:rbf-c}
\end{figure*}

\subsection{Impact of Number of Class Labels.}
\textbf{\textit{Setup.}} To investigate the sole impact of the number of class labels of the dataset, i.e. the $p$ value, on the accuracy of an ensemble, we use the GOOWE \cite{bib:bonabtkdd} method. It uses our optimum weight vector calculation for vote aggregation in a WMV procedure for vote aggregation. The Hoeffding Tree (HT) \cite{bib:ht} is used as the component classifier, due to its high adaptivity to data stream classification. For a fair comparison, we modify GOOWE for having MV aggregation rule by simply providing equal weights to the components. These two variations of GOOWE, i.e. WMV and MV, are used for this experiment. Each of these variations trained and tested using different ensemble size values, starting from only two components and doubling at each step---i.e. our investigated ensemble sizes, $m$ values, are $2, 4, \cdots, 128$.

\textbf{\textit{Dataset.}} Since existing real-world datasets are not consistent, in terms of classification complexity, we are only able to use synthetic data for this experiment in order to have reasonable comparisons. We choose the popular Random RBF generator, since it is capable of generating data streams with an arbitrary number of features and class labels \cite{bib:2009}. Using this generator, implemented in the MOA framework \cite{bib:moa}, we prepare six datasets, each containing one million instances with $20$ attributes, with the default parameter settings of the RBF generator. The only difference is the number of class labels among datasets which are $2, 4, 8, 16, 32, \text{~and~} 64$. We reflect this in dataset naming as RBF-C2, RBF-C4, $\cdots$, respectively. 
 
\textbf{\textit{Results.}} Fig. \ref{fig:rbf-c} presents prediction accuracy for WMV and MV with increasing component counts, $m$, on each dataset. To mark the equality of $m$ and $p$, we use a vertical dashed green line. We can make the following interesting observations:
\begin{enumerate*}[label=(\roman*)]
    \item A weighted aggregation rule becomes more vital with an increasing number of component classifiers,
    \item WMV is performing more resiliently in multi-class problems, compared to binary classification problems, when compared to MV. The gap between WMV and MV seems to increase with greater numbers of class labels,
    \item There is a peak point in the accuracy value, and it is dependent on the number of class labels. This can be seen by comparing RBF-C2, RBF-C4, and RBF-C8 (the first row in Fig. \ref{fig:rbf-c}) with RBF-C16, RBF-C32, and RBF-C64 (the second row in Fig. \ref{fig:rbf-c}) plots. In the former set we see that after a peak point, the accuracy starts to drop. However, in the latter set we see that the peak points are with $m=128$, and
    \item The theoretical vertical line, i.e. the equality case $m=p$, seems to precede the peak point on each plot. We suspect that this might be due to Theorem 4's premise conditions: generating independent scores and aggregating with optimum weights.  
\end{enumerate*}

\subsection{Impact of Data Streams.}
\textbf{\textit{Setup.}} There are many factors when the complexity of classification problems are considered---concept drift, the number of features, etc. To this end, we investigate the number of component classifiers for WMV and MV on a wide range of datasets. We use an experimental setup similar to the previous experiments on different synthetic and real-world datasets. We aim to investigate some general patterns in more realistic problems.  

\textbf{\textit{Dataset.}} We select eight synthetic and eight real-world benchmark datasets used for stream classification problems in the literature. A summary of our datasets is given in Table \ref{tab:dataset}. For this selection, we aim to have a mixture of different concept drift types, number of features, number of class labels, and noise percentages. Synthetic datasets are similar to the ones used for the GOOWE evaluation \cite{bib:bonabtkdd}. For real-world datasets, Sensor, PowerSupply, and HyperPlane datasets are taken from\footnote{Access URL: http://www.cse.fau.edu/$\sim$xqzhu/stream.html}. The remainder of real-world datasets are taken from\footnote{Access URL: https://www.openml.org/}. See \cite{bib:bonabcikm,bib:bonabtkdd} for a detailed explanations of the datasets. 

\begin{figure*}
	\includegraphics[width=\textwidth]{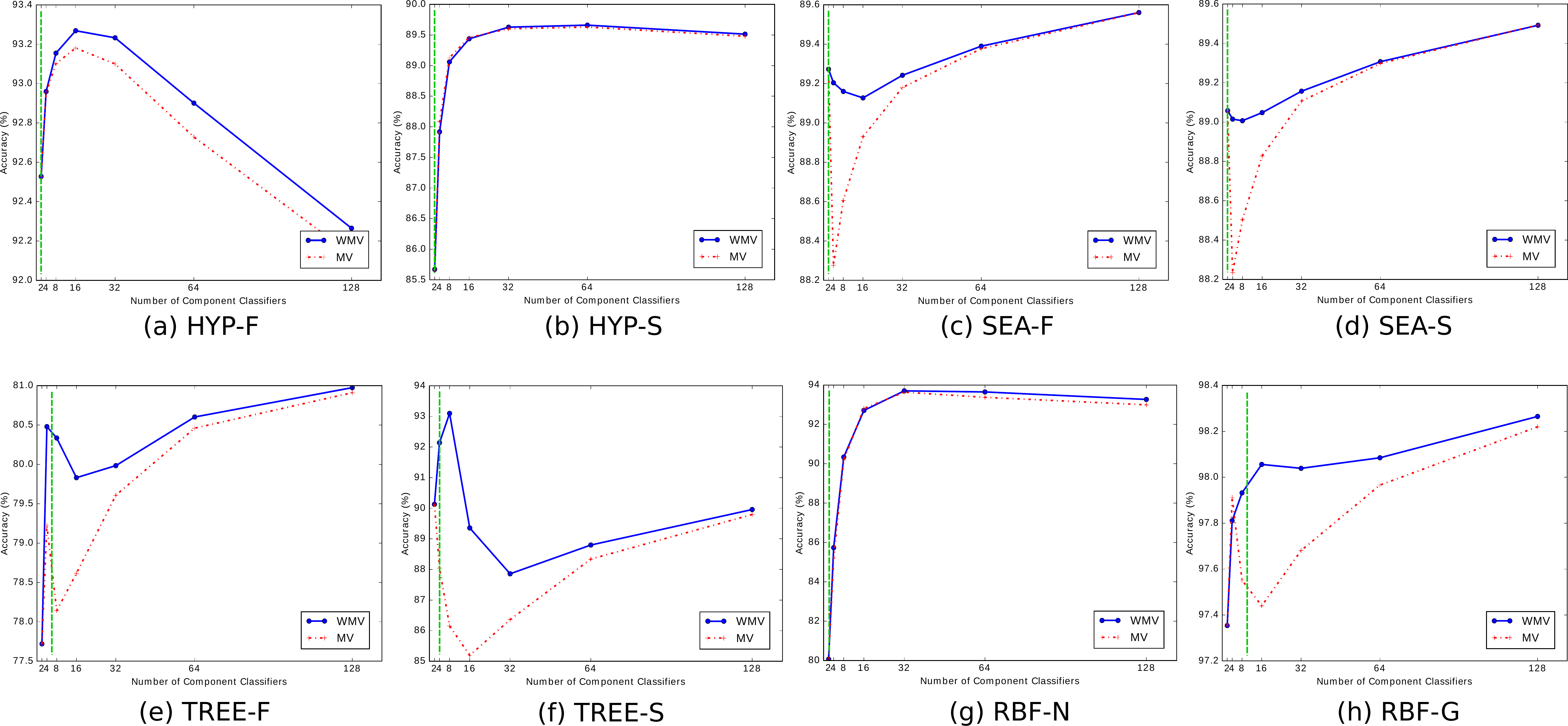}
	\caption{Prediction behavior of WMV and MV aggregation rules, in terms of accuracy, for miscellaneous synthetic datasets, with increasing both the number of component classifiers, $m$, and class labels, $p$. The equality case, $m=p$, is marked on each plot using a vertical dashed green line.  } 
	\label{fig:synth}
\end{figure*}

\begin{table}
    \centering
    \caption{Summary of Dataset Characteristics}
    \vspace{-0.3cm}
	\begin{tabular}{lccccc} 
		\toprule 
		Dataset & \#Instance & \#Att & \#CL & \%N & Drift Spec. \\ 
		\hline 
		HYP-F & \( 1\times 10^6 \) & 10 & 2 & 5 & Incrm., DS=0.1\\
		HYP-S & \( 1\times 10^6 \) & 10 & 2 & 5 & Incrm., DS=0.001\\
		SEA-F & \( 1\times 10^6 \) & 3 & 2 & 10 & Abrupt, \#D=9\\
		SEA-S & \( 1\times 10^6 \) & 3 & 2 & 10 & Abrupt, \#D=3\\
		TREE-F & \( 1\times 10^6 \) & 10 & 6 & 0 & Reoc., \#D=15\\
		TREE-S & \( 1\times 10^6 \) & 10 & 4 & 0 & Reoc., \#D=4\\
		RBF-G  & \( 1\times 10^6 \) & 20 & 10 & 0 & Gr., DS=0.01\\
		RBF-N  & \( 1\times 10^6 \) & 20 & 2 & 0 & No Drift\\
		\midrule
		Airlines & 539,383 & 7 & 2 & - & Unknown\\
		ClickPrediction & 399,482 & 11 & 2 & - & Unknown\\
		Electricity & 45,311 & 8 & 2  & - & Unknown\\
		HyperPlane & \( 1\times 10^5 \) & 10 & 5  & - & Unknown\\
        CoverType & 581,012 & 54 & 7 & - & Unknown\\
		PokerHand & \( 1\times 10^7 \) & 10 & 10 & - & Unknown\\
		PowerSupply & 29,925 & 2 & 24 & - & Unknown\\
		Sensor  & 2,219,802 & 5 & 58 & - & Unknown\\
		\bottomrule 
		\vspace{-0.1cm}
	\end{tabular}
    \raggedright Note: {\#CL:} {No. of Class Labels,} {\%N:} {Percentage of Noise,} {DS:} {Drift Speed,} {\#D:} {No. of Drifts,} {Gr.:} {Gradual}.
    \label{tab:dataset}
\end{table}

\textbf{\textit{Results.}}  Fig. \ref{fig:synth} and \ref{fig:reald} present the prediction accuracy difference for WMV and MV for increasing component classifier counts, $m$, on each dataset. For marking the equality of $m$ and $p$, we use a vertical dashed green line, similar to the previous experiments. As we can see, given more broad types of datasets, each with completely different complexities, it is difficult to conclude strict patterns. We have the following interesting observations:
\begin{enumerate*}[label=(\roman*)]
    \item For almost all the datasets, WMV, with optimum weights, outperforms MV, 
    \item We can see the same results as the previous experiments: There is a peak point in the accuracy value and it is dependent on the number of class labels, 
    \item The theoretical vertical line, i.e. the equality case $m=p$, seems to precede the peak point on each plot, and
    \item Optimum weighting seems to be more resilient in the evolving environments, i.e. data streams with concept drift, regardless of the type of concept drift. 
\end{enumerate*}

The observations we have with the real-world data streams provide strong evidence that supports our claim which indicates that the number of class labels has an important influence on the ideal number of component classifiers and prediction performance. In Fig. \ref{fig:reald}, we observe that the peak performances, with one exception, are not observed with the maximum ensemble size. In other words, as we increase the number of component classifiers and move away from the green line and employ an ensemble of size 128, in all cases, prediction performance becomes lower than that of a smaller size ensemble. The only exception is observed with ClickPrediction; even with that one, no noticeable improvement is provided with the largest ensemble size. Furthermore, in all data streams, except ClickPrediction, the peak performances are closer to the green line rather than being closer to the largest ensemble size.

\begin{figure*}
	\includegraphics[width=\textwidth]{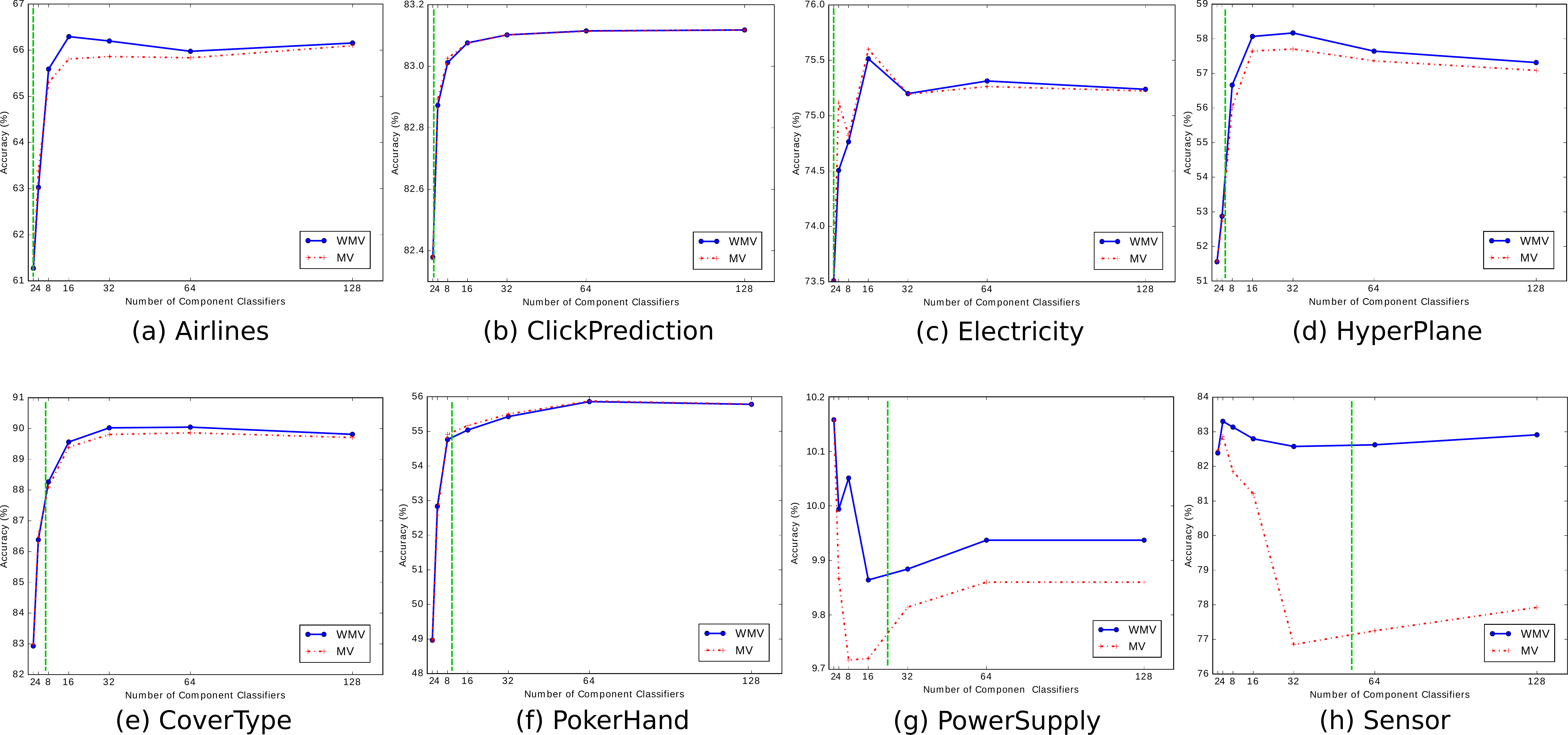}
	\caption{Prediction behavior of WMV and MV aggregation rules, in terms of accuracy, for miscellaneous real-world datasets with increasing both the number of component classifiers, $m$, and class labels, $p$. The equality case, $m=p$, is marked on each plot using a vertical dashed green line. } 
	\label{fig:reald}
\end{figure*}

\subsection{Impact of Diversity.}
 \textbf{\textit{Setup.}} In order to study the impact of diversity in ensemble classifiers and show the practical value of our theorems, we design two different scenarios for the binary classification problem. We select binary classification for the purpose of this experiment since the difference between WMV and MV are almost always insignificant for binary classification, compared to multi-class problems. In addition, multi-class problems can potentially modeled as several binary classification problems \cite{rifkin2004}.  
 
 To this end, we recruit a well-known and state-of-the-art online ensemble classifier, called leverage bagging (LevBag) \cite{bib:lev} as the base ensemble method for our comparisons. It is based on the OzaBagging ensemble \cite{bib:ozaphd,bib:oza2} and is proven to react well in online environments. It exploits re-sampling with replacement (i.e. input randomization), using a $Poisson(\lambda)$ distribution to train diversified component classifiers. 
 
 We use LevBag in our experiments since it initializes a fixed number of component classifiers---i.e. unlike GOOWE, where component classifiers are dynamically added and removed during training in the course of incoming stream data \cite{bib:bonabtkdd}, for LevBag the number of component classifiers are fixed from the initialization and the ensemble does not alter them. In addition, LevBag uses error-correcting output codes for handling multi-class problems, and transforms them into several binary classification problems \cite{bib:lev}. Majority voting is used for vote aggregation, as the baseline of our experiments.

 \textbf{\textit{Design.}} For our analysis, we train different LevBag ensembles with 2, 4, ..., 64, and 128 components of classifiers---named LevBag-2, LevBag-4, $\cdots$, respectively. The Hoeffding Tree (HT) \cite{bib:ht} is used as the component classifier.

We design two experimental scenarios and compare them with LevBag ensembles as baselines. Each scenario is designed to show the practical value of our theorems with different perspectives. Here is a brief description.

\begin{itemize}
	\item \textbf{Scenario 1.} We select the two most diverse components out of a LevBag-10 ensemble's pool of component classifiers, called Sel2Div ensemble, and aggregate their votes. For pairwise diversity measurements among the components of the ensemble, the Yule's Q-statistic \cite{kuncheva2003measures} is used. Minku et al. \cite{bib:minku1} used it for pairwise diversity measurements of online ensemble learning. Q-statistic is measured between all the pairs, and the highest diverse pair is chosen. For two classifiers $CS_r$ and $CS_s$, the Q-statistic is defined as below. $N^{ab}$ is the number of instances in the instance window that $CS_r$ predicts $a$ and $CS_s$ predicts $b$.  
	\[
		Q_{r,s} = \dfrac{ N^{11}N^{00} - N^{01}N^{10} }{N^{11}N^{00} + N^{01}N^{10}}
	\]
	\item \textbf{Scenario 2.} We train a hybrid of two different algorithms as component classifiers, a potentially diverse ensemble. For this, one instance of the Hoeffding Tree (HT) and the Naive Bayes (NB) \cite{bib:2009} algorithms are exploited; both are trained on the same instances of data stream---without input randomization. We call this the Hyb-HTNB ensemble. 
\end{itemize}

\begin{table*}[t]\centering
	\caption{Classification Accuracy in Percentage (\%)---the highest accuracy for each dataset is bold} 
	\label{table:expres}
	\addtolength{\tabcolsep}{-2pt} 
	\vspace{-0.3cm}
	\begin{tabular}{lccccccccc}
		\toprule
		& \multicolumn{7}{c}{LevBag} & \multicolumn{1}{c}{Select 2 most diverse} & \multicolumn{1}{c}{Hybrid of HT and NB}\\
		\cmidrule(lr){2-8}  \cmidrule(lr){9-9} \cmidrule(lr){10-10}
		Dataset         & 2 & 4 & 8 & 16 & 32 & 64 & 128 & 2 & 2  \\
		\midrule
		
		Airlines            & 85.955 & 86.747  & 87.455 & 88.136  & 88.527 & 89.516 & \textbf{90.638} & 88.430 & 86.392 \\
		ClickPrediction          & 95.395 & 95.516  & 95.515 & 95.531  & \textbf{95.532} & 95.524 & 95.525 & 95.520 & 95.524 \\
		Electricity          & 83.481 & 84.299  & 84.172 & 84.842  & 84.332 & 84.797 & 84.906 & \textbf{85.406} & 82.538 \\
		\midrule
		RBF            & 78.391 & 79.210  & 79.750 & 79.789  & 80.065 & 80.321 &\textbf{ 80.339} & 78.575 & 77.483 \\
		SEA            & 86.011 & \textbf{86.354} & 86.070 & 86.246  & 85.387 & 84.976 & 84.872 & 86.166 & 84.650 \\
		HYP            & 87.620 & 87.957  & \textbf{88.839} & 88.388  & 88.292 & 88.176 & 88.313 & 87.957 & 88.283 \\
		\bottomrule	
	\end{tabular}
\end{table*}

\begin{table}[!t]
	\caption{The Multiple Comparisons for Friedman Statistical Test Results. Minimum required difference of mean rank is \textbf{2.635}. Higher mean rank means better performance. } 
	\label{tab:stat}
	\centering 
	\vspace{-0.3cm}
	\begin{tabular}{llcl} 
		\toprule 
		Idx & Ensemble & Mean Rank & Different From ($P<0.05$)  \\ 
		\midrule
		\textbf{(1)}& LevBag-2  & 2.000 & (3)(4)(5)(6)(7)(8) \\
		\textbf{(2)}& LevBag-4  & 4.250 & (4)(7) \\
		\textbf{(3)}& LevBag-8  & 4.833 & (1) \\		
		\textbf{(4)}& LevBag-16 & 7.000 & (1)(2)(9) \\
		\textbf{(5)}& LevBag-32 & 6.333 & (1)(9) \\		
		\textbf{(6)}& LevBag-64 & 5.750 & (1)(9) \\		
		\textbf{(7)}& LevBag-128& 7.000 & (1)(2)(9) \\		
		\textbf{(8)}& Sel2Div   & 5.250 & (1)(9) \\
		\textbf{(9)}& Hyb-HTNB  & 2.583 & (4)(5)(6)(7)(8) \\
		\bottomrule 
	\end{tabular}
\end{table}

The ensemble sizes for both of these scenarios are two. For each instance, vote aggregation in both scenarios is done using our geometric weighting framework. An instance window of 100 latest incoming instances are kept, and using Eq. \ref{eq:aqj1} and \ref{eq:dq1} weights are calculated---$w\Lambda=\gamma$. 

 \textbf{\textit{Dataset.}}  We examine our experiments using three real-world and three synthetic data streams, all with two-class labels. For real-world datasets, we use the exact same real-world datasets with two class labels as with the previous experiments, see Table \ref{tab:dataset}. For synthetic datasets, we generate 500,000 instances of RBF, SEA, and HYP stream generator from the MOA framework \cite{bib:moa}. All the setting of these generators are set to the default, except for the number of class labels, which is two.

\textbf{\textit{Results.}} 
Table \ref{table:expres} shows the prediction accuracy of different ensemble sizes and experimental scenarios for examined datasets. The highest accuracy for each dataset are bold. We can see that the ensemble size and component selection has a crucial impact on accuracy of prediction. 

o differentiate the significance of differences in accuracy values, we exploited the non-parametric Friedman statistical test, with $\alpha=0.05$ and $F(8,40)$. The \textit{null-hypothesis} for this statistical test claims that there is no statistically significant difference among all examined ensembles, in terms of accuracy. The resulting two-tailed probability value, $P=0.002$, rejects the null-hypothesis and shows that the differences are significant.

The Friedman multiple pairwise comparisons are conducted and presented in Table \ref{tab:stat}. It can be seen that there is no significant difference among LevBag-8, LevBag-16, LevBag-32, LevBag-64, LevBag-128, and Sel2Div ensembles. Given that all are trained using the same component classifier, the impact of this result is important; only 2 base classifiers can be comparably good with 128 of them, when they trained in a diverse enough fashion and weighted optimally.

On the other hand, the Hyb-HTNB ensemble performs equivalently as good as LevBag-2, LevBag-4, and LevBag-8, according to statistical significance tests. Hyb-HTNB is a naturally diverse ensemble; we included this in our experiment to show the impact of diversity on prediction accuracy. Since NB is a weak classifier compared to HT, it is reasonable that Hyb-HTNB is not performing as good as the Sel2Div ensemble.

\section{Conclusion}
In this paper, we studied the impact of ensemble size using a geometric framework. The entire decision making process through voting is adapted to a spatial environment and weighting combination rules, including majority voting, are considered for providing better insight. The focus of study is on online ensembles, however nothing prevents us from using the proposed model on batch ensembles. 


The ensemble size is crucial for online environments, due to the dimensionality growth of data. We discussed the effect of ensemble size with majority voting and optimal weighted voting aggregation rules. The highly important outcome is that we do not need to train a near-infinite number of components to have a good ensemble. 

We delivered a framework which heightens the understanding of the diversity, and explains why diversity contributes to the accuracy of predictions. 


Our experimental evaluations showed the practical value of our theorems, and highlighted existing challenges. Practical imperfections across different algorithms and different learning complexities on our various datasets prevent us to clearly show that $m = p$ and diversity are the core decisions to be used in the ensemble design. However, the experimental results show that the number of class labels has an important effect on the ensemble size. For example, in 7 out of 8 real world datasets, the peak performances are closer to the ideal $m = p$ point rather than being closer to the largest ensemble size.

As a future work, we aim to define some diversity measures based on this framework, while also studying the coefficient matrix specifications.

\balance


%

%


\ifCLASSOPTIONcaptionsoff
  \newpage
\fi



%

\bibliographystyle{IEEEtran}
\bibliography{biblist}

\begin{thebibliography}{10}
\providecommand{\url}[1]{#1}
\csname url@samestyle\endcsname
\providecommand{\newblock}{\relax}
\providecommand{\bibinfo}[2]{#2}
\providecommand{\BIBentrySTDinterwordspacing}{\spaceskip=0pt\relax}
\providecommand{\BIBentryALTinterwordstretchfactor}{4}
\providecommand{\BIBentryALTinterwordspacing}{\spaceskip=\fontdimen2\font plus
\BIBentryALTinterwordstretchfactor\fontdimen3\font minus
  \fontdimen4\font\relax}
\providecommand{\BIBforeignlanguage}[2]{{%
\expandafter\ifx\csname l@#1\endcsname\relax
\typeout{** WARNING: IEEEtran.bst: No hyphenation pattern has been}%
\typeout{** loaded for the language `#1'. Using the pattern for}%
\typeout{** the default language instead.}%
\else
\language=\csname l@#1\endcsname
\fi
#2}}
\providecommand{\BIBdecl}{\relax}
\BIBdecl

\bibitem{bib:bonabcikm}
H.~Bonab and F.~Can, ``A theoretical framework on the ideal number of
  classifiers for online ensembles in data streams,'' in \emph{Proc. of
  Conference on Information and Knowledge Management (CIKM)}.\hskip 1em plus
  0.5em minus 0.4em\relax ACM, 2016, pp. 2053--2056.

\bibitem{can1993incremental}
F.~Can, ``Incremental clustering for dynamic information processing,''
  \emph{ACM Transactions on Information Systems (TOIS)}, vol.~11, no.~2, pp.
  143--164, 1993.

\bibitem{dietterich2000ensemble}
T.~G. Dietterich, ``Ensemble methods in machine learning,'' in \emph{Multiple
  Classifier Systems (MCS)}.\hskip 1em plus 0.5em minus 0.4em\relax Springer,
  2000, pp. 1--15.

\bibitem{bib:gama2010book}
J.~Gama, \emph{Knowledge discovery from data streams}.\hskip 1em plus 0.5em
  minus 0.4em\relax CRC Press, 2010.

\bibitem{bib:wang2015}
S.~Wang, L.~L. Minku, and X.~Yao, ``Resampling-based ensemble methods for
  online class imbalance learning,'' \emph{{IEEE} Transactions on Knowledge and
  Data Engineering (TKDE)}, vol.~27, no.~5, pp. 1356--1368, 2015.

\bibitem{bib:surveygama2017}
B.~Krawczyk, L.~L. Minku, J.~Gama, J.~Stefanowski, and M.~Wo{\'z}niak,
  ``Ensemble learning for data stream analysis: A survey,'' \emph{Information
  Fusion}, vol.~37, pp. 132--156, 2017.

\bibitem{bib:rokach2010ensemble}
L.~Rokach, ``Ensemble-based classifiers,'' \emph{Artificial Intelligence
  Review}, vol.~33, no.~1, pp. 1--39, 2010.

\bibitem{bib:chan1999weighted}
L.-W. Chan, ``Weighted least square ensemble networks,'' in \emph{International
  Joint Conference on Neural Networks (IJCNN)}, vol.~2.\hskip 1em plus 0.5em
  minus 0.4em\relax IEEE, 1999, pp. 1393--1396.

\bibitem{bib:latinne2001limiting}
P.~Latinne, O.~Debeir, and C.~Decaestecker, ``Limiting the number of trees in
  random forests,'' in \emph{Multiple Classifier Systems (MCS)}.\hskip 1em plus
  0.5em minus 0.4em\relax Springer, 2001, pp. 178--187.

\bibitem{hu2001using}
X.~Hu, ``Using rough sets theory and database operations to construct a good
  ensemble of classifiers for data mining applications,'' in \emph{Proc. of
  International Conference on Data Mining (ICDM)}.\hskip 1em plus 0.5em minus
  0.4em\relax IEEE, 2001, pp. 233--240.

\bibitem{bib:oshiro2012many}
T.~M. Oshiro, P.~S. Perez, and J.~A. Baranauskas, ``How many trees in a random
  forest?'' in \emph{Machine Learning and Data Mining in Pattern Recognition
  (MLDM)}.\hskip 1em plus 0.5em minus 0.4em\relax Springer, 2012, pp. 154--168.

\bibitem{bib:hernandez2013large}
D.~Hern{\'a}ndez-Lobato, G.~Mart{\'\i}nez-Mu{\~n}oz, and A.~Su{\'a}rez, ``How
  large should ensembles of classifiers be?'' \emph{Pattern Recognition},
  vol.~46, no.~5, pp. 1323--1336, 2013.

\bibitem{bib:surveybifet2017}
H.~M. Gomes, J.~P. Barddal, F.~Enembreck, and A.~Bifet, ``A survey on ensemble
  learning for data stream classification,'' \emph{ACM Computing Surveys},
  vol.~50, no.~2, pp. 23:1--23:36, Mar. 2017.

\bibitem{tsoumakas2008taxonomy}
G.~Tsoumakas, I.~Partalas, and I.~Vlahavas, ``A taxonomy and short review of
  ensemble selection,'' in \emph{Supervised and Unsupervised Ensemble Methods
  and Their Applications}, 2008, pp. 41--46.

\bibitem{gomes2014mining}
J.~B. Gomes, M.~M. Gaber, P.~A. Sousa, and E.~Menasalvas, ``Mining recurring
  concepts in a dynamic feature space,'' \emph{{IEEE} Transactions on Neural
  Networks and Learning Systems (TNNLS)}, vol.~25, no.~1, pp. 95--110, 2014.

\bibitem{bib:jackowski2018new}
K.~Jackowski, ``New diversity measure for data stream classification
  ensembles,'' \emph{Engineering Applications of Artificial Intelligence},
  vol.~74, pp. 23--34, 2018.

\bibitem{kuncheva2004combining}
L.~I. Kuncheva, \emph{Combining pattern classifiers: Methods and
  algorithms}.\hskip 1em plus 0.5em minus 0.4em\relax John Wiley \& Sons, 2004.

\bibitem{kuncheva2003measures}
L.~I. Kuncheva and C.~J. Whitaker, ``Measures of diversity in classifier
  ensembles and their relationship with the ensemble accuracy,'' \emph{Machine
  Learning}, vol.~51, no.~2, pp. 181--207, 2003.

\bibitem{liu2004empirical}
H.~Liu, A.~Mandvikar, and J.~Mody, ``An empirical study of building compact
  ensembles,'' in \emph{Web-Age Information Management (WAIM)}.\hskip 1em plus
  0.5em minus 0.4em\relax Springer, 2004, pp. 622--627.

\bibitem{brown2005diversity}
G.~Brown, J.~Wyatt, R.~Harris, and X.~Yao, ``Diversity creation methods: A
  survey and categorisation,'' \emph{Information Fusion}, vol.~6, no.~1, pp.
  5--20, 2005.

\bibitem{bib:lsqbook}
P.~C. Hansen, V.~Pereyra, and G.~Scherer, \emph{Least squares data fitting with
  applications}.\hskip 1em plus 0.5em minus 0.4em\relax Johns Hopkins
  University Press, 2013.

\bibitem{zhou2002ensembling}
Z.-H. Zhou, J.~Wu, and W.~Tang, ``Ensembling neural networks: Many could be
  better than all,'' \emph{Artificial Intelligence}, vol. 137, no. 1-2, pp.
  239--263, 2002.

\bibitem{ulacs2009incremental}
A.~Ula{\c{s}}, M.~Semerci, O.~T. Y{\i}ld{\i}z, and E.~Alpayd{\i}n,
  ``Incremental construction of classifier and discriminant ensembles,''
  \emph{Information Sciences}, vol. 179, no.~9, pp. 1298--1318, 2009.

\bibitem{rokach2009collective}
L.~Rokach, ``Collective-agreement-based pruning of ensembles,''
  \emph{Computational Statistics \& Data Analysis}, vol.~53, no.~4, pp.
  1015--1026, 2009.

\bibitem{margineantu1997pruning}
D.~D. Margineantu and T.~G. Dietterich, ``Pruning adaptive boosting,'' in
  \emph{International Conference on Machine Learning (ICML)}, vol.~97, 1997,
  pp. 211--218.

\bibitem{toraman2012}
C.~Toraman and F.~Can, ``Squeezing the ensemble pruning: Faster and more
  accurate categorization for news portals,'' in \emph{European Conference on
  Information Retrieval ({ECIR})}.\hskip 1em plus 0.5em minus 0.4em\relax
  Springer, 2012, pp. 508--511.

\bibitem{bib:nse}
R.~Elwell and R.~Polikar, ``Incremental learning of concept drift in
  nonstationary environments,'' \emph{{IEEE} Transactions on Neural Networks
  and Learning Systems (TNNLS)}, vol.~22, no.~10, pp. 1517--1531, 2011.

\bibitem{windeatt2013ensemble}
T.~Windeatt and C.~Zor, ``Ensemble pruning using spectral coefficients,''
  \emph{{IEEE} Transactions on Neural Networks and Learning Systems (TNNLS)},
  vol.~24, no.~4, pp. 673--678, 2013.

\bibitem{bib:decbound}
K.~Tumer and J.~Ghosh, ``Analysis of decision boundaries in linearly combined
  neural classifiers,'' \emph{Pattern Recognition}, vol.~29, no.~2, pp.
  341--348, 1996.

\bibitem{bib:fumera-ref2}
E.~Bauer and R.~Kohavi, ``An empirical comparison of voting classification
  algorithms: Bagging, boosting, and variants,'' \emph{Machine Learning},
  vol.~36, no. 1-2, pp. 105--139, 1999.

\bibitem{bib:fumera}
G.~Fumera, F.~Roli, and A.~Serrau, ``A theoretical analysis of bagging as a
  linear combination of classifiers,'' \emph{{IEEE} Transactions on Pattern
  Analysis and Machine Intelligence (TPAMI)}, vol.~30, no.~7, pp. 1293--1299,
  2008.

\bibitem{bib:fumera-ref8}
G.~Fumera and F.~Roli, ``A theoretical and experimental analysis of linear
  combiners for multiple classifier systems,'' \emph{{IEEE} Transactions on
  Pattern Analysis and Machine Intelligence (TPAMI)}, vol.~27, no.~6, pp.
  942--956, 2005.

\bibitem{pietruczuk2016method}
L.~Pietruczuk, L.~Rutkowski, M.~Jaworski, and P.~Duda, ``A method for automatic
  adjustment of ensemble size in stream data mining,'' in \emph{International
  Joint Conference on Neural Networks (IJCNN)}.\hskip 1em plus 0.5em minus
  0.4em\relax IEEE, 2016, pp. 9--15.

\bibitem{pietruczuk2017adjust}
------, ``How to adjust an ensemble size in stream data mining?''
  \emph{Information Sciences}, vol. 381, pp. 46--54, 2017.

\bibitem{bib:bonabtkdd}
H.~Bonab and F.~Can, ``{GOOWE}: Geometrically optimum and online-weighted
  ensemble classifier for evolving data streams,'' \emph{ACM Transactions on
  Knowledge Discovery from Data (TKDD)}, vol.~12, no.~2, pp. 25:1--25:33, Mar.
  2018.

\bibitem{bib:dwm}
J.~Z. Kolter and M.~A. Maloof, ``Dynamic weighted majority: An ensemble method
  for drifting concepts,'' \emph{Journal of Machine Learning Research (JMLR)},
  vol.~8, pp. 2755--2790, 2007.

\bibitem{bib:dwm2}
------, ``Dynamic weighted majority: A new ensemble method for tracking concept
  drift,'' in \emph{Proc. of International Conference on Data Mining
  (ICDM)}.\hskip 1em plus 0.5em minus 0.4em\relax IEEE, 2003, pp. 123--130.

\bibitem{miller1999critic}
D.~J. Miller and L.~Yan, ``Critic-driven ensemble classification,'' \emph{IEEE
  Transactions on Signal Processing}, vol.~47, no.~10, pp. 2833--2844, 1999.

\bibitem{wu2015geometric}
S.~Wu and F.~Crestani, ``A geometric framework for data fusion in information
  retrieval,'' \emph{Information Systems}, vol.~50, pp. 20--35, 2015.

\bibitem{abramowitz1964handbook}
M.~Abramowitz and I.~A. Stegun, \emph{Handbook of mathematical functions with
  formulas, graphs, and mathematical tables}.\hskip 1em plus 0.5em minus
  0.4em\relax Courier Corporation, 1964, vol.~55.

\bibitem{tumer1996error}
K.~Tumer and J.~Ghosh, ``Error correlation and error reduction in ensemble
  classifiers,'' \emph{Connection Science}, vol.~8, no. 3-4, pp. 385--404,
  1996.

\bibitem{bib:awe}
H.~Wang, W.~Fan, P.~S. Yu, and J.~Han, ``Mining concept-drifting data streams
  using ensemble classifiers,'' in \emph{Proc. of International Conference on
  Knowledge Discovery and Data Mining ({SIGKDD})}.\hskip 1em plus 0.5em minus
  0.4em\relax ACM, 2003, pp. 226--235.

\bibitem{bib:moa}
A.~Bifet, G.~Holmes, R.~Kirkby, and B.~Pfahringer, ``{MOA:} massive online
  analysis,'' \emph{Journal of Machine Learning Research (JMLR)}, vol.~11, pp.
  1601--1604, 2010.

\bibitem{bib:ht}
P.~M. Domingos and G.~Hulten, ``Mining high-speed data streams,'' in
  \emph{Proc. of International Conference on Knowledge Discovery and Data
  Mining ({SIGKDD})}.\hskip 1em plus 0.5em minus 0.4em\relax ACM, 2000, pp.
  71--80.

\bibitem{bib:2009}
A.~Bifet, G.~Holmes, B.~Pfahringer, R.~Kirkby, and R.~Gavald{\`{a}}, ``New
  ensemble methods for evolving data streams,'' in \emph{Proc. of International
  Conference on Knowledge Discovery and Data Mining ({SIGKDD})}, 2009, pp.
  139--148.

\bibitem{rifkin2004}
R.~Rifkin and A.~Klautau, ``In defense of one-vs-all classification,''
  \emph{Journal of Machine Learning Research (JMLR)}, vol.~5, pp. 101--141,
  2004.

\bibitem{bib:lev}
A.~Bifet, G.~Holmes, and B.~Pfahringer, ``Leveraging bagging for evolving data
  streams,'' in \emph{Proc. of International Conference on Machine Learning and
  Knowledge Discovery in Databases ({ECML-PKDD})}, 2010, pp. 135--150.

\bibitem{bib:ozaphd}
N.~C. Oza, ``Online ensemble learning,'' Ph.D. dissertation, Computer Science
  Division, Univ. California, Berkeley, CA, USA, 09 2001.

\bibitem{bib:oza2}
N.~C. Oza and S.~Russell, ``Experimental comparisons of online and batch
  versions of bagging and boosting,'' in \emph{Proc. of International
  Conference on Knowledge Discovery and Data Mining ({SIGKDD})}.\hskip 1em plus
  0.5em minus 0.4em\relax ACM, 2001, pp. 359--364.

\bibitem{bib:minku1}
L.~L. Minku, A.~P. White, and X.~Yao, ``The impact of diversity on online
  ensemble learning in the presence of concept drift,'' \emph{{IEEE}
  Transactions on Knowledge and Data Engineering (TKDE)}, vol.~22, no.~5, pp.
  730--742, 2010.

\end{thebibliography}

%

\begin{IEEEbiography}[{\includegraphics[width=1in,height=1.25in,clip,keepaspectratio]{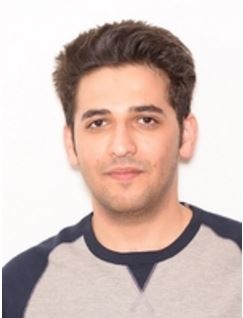}}]{Hamed Bonab}
received BS and MS degrees in computer engineering from the Iran University of Science and Technology (IUST), Tehran-Iran, and Bilkent University, Ankara-Turkey, respectively. Currently, he is a PhD student in the College of Information and Computer Sciences at the University of Massachusetts Amherst, MA. His research interests broadly include stream processing, data mining, machine learning, and information retrieval.
\end{IEEEbiography}

\begin{IEEEbiography}[{\includegraphics[width=1in,height=1.25in,clip,keepaspectratio]{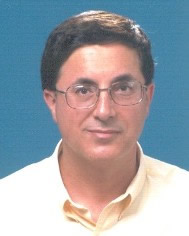}}]{Fazli Can} received PhD degree (1985) in computer engineering, from the Middle East Technical University (METU), Ankara, Turkey. His BS and MS degrees, respectively, on electrical \& electronics and computer engineering are also from METU. He conducted his PhD research under the supervision of Prof. Esen Ozkarahan; at Arizona State University (ASU) and Intel, AZ; as a part of the RAP database machine project. He is a faculty member at Bilkent University, Ankara, Turkey. Before joining Bilkent he was a tenured full professor at Miami University, Oxford, OH. He co-edited ACM SIGIR Forum (1995-2002) and is a co-founder of the Bilkent Information Retrieval Group. His research interests are information retrieval and data mining. His interest in dynamic information processing dates back to his 1993 incremental clustering paper in ACM TOIS.
\end{IEEEbiography}



\end{document}